\documentclass[12pt]{article}

\usepackage[utf8]{inputenc}
\usepackage[T1]{fontenc}
\usepackage{lmodern}
\usepackage{amsmath,amssymb,amsthm}
\usepackage{geometry}
\usepackage{setspace}
\usepackage{hyperref}
\usepackage{enumitem}

\title{A Long-Run Persistence Theory for AI Systems under the Redundancy-Adjusted Artificial Age Score (AAS)\\[1.5cm]}
\author{
Seyma Yaman Kayadibi\\
Victoria University\\
\texttt{seyma.yamankayadibi@live.vu.edu.au}
}
\date{}

\newtheorem{theorem}{Theorem}[section]
\newtheorem{lemma}[theorem]{Lemma}
\newtheorem{corollary}[theorem]{Corollary}
\newtheorem{definition}[theorem]{Definition}

\begin{document}

\maketitle

\begin{abstract}
Artificial intelligence systems are increasingly expected to operate over repeated cycles of interaction, adaptation, and update rather than through isolated one-shot outputs. This shift raises a fundamental theoretical question: can an AI system persist indefinitely without incurring unbounded structural aging? This paper develops a long-run persistence framework for AI systems based on the redundancy-adjusted Artificial Age Score (AAS). The proposed model extends AAS from a static evaluative measure into a cycle-level functional that generates an age sequence across repeated operation. For each cycle, structural age is defined through a weighted, redundancy-aware logarithmic penalty over component consistency levels. Within this formulation, cycle-level age is shown to be well defined and uniformly bounded, which excludes explosive pointwise aging. Building on this foundation, the paper defines a hierarchy of asymptotic regimes, including burdened persistence, zero-burden persistence, oscillatory persistence, and cumulative terminal burden. It further establishes comparative ordering, sensitivity bounds, convergence under componentwise stabilization, persistence under finite total variation, geometric stabilization under damped inter-cycle perturbations, and a zero-burden characterization under nondegenerate redundancy conditions. The central result is that, within the proposed framework, indefinite cyclic continuation does not require unbounded aging: an AI system may pass through infinitely many cycles while its structural age remains bounded, and under stronger regularity conditions its marginal aging vanishes and, in the strongest regime, its cycle-level burden converges to zero. The framework therefore provides a formal basis for analyzing long-run artificial persistence as a problem of bounded structural burden rather than inevitable cumulative deterioration.
\end{abstract}

\noindent\textbf{Keywords:} Artificial Age Score; long-run persistence; structural burden; redundancy adjustment; asymptotic regimes; bounded aging; repeated artificial operation; cycle-level analysis

\section{Introduction}

\subsection{Repeated artificial operation as a theoretical problem}

Artificial intelligence is increasingly evaluated not only by whether it can produce a satisfactory output on a single occasion, but also by how it behaves under repeated conditions of use. This broader orientation is consistent with the operational turn introduced by Turing, who redirected the question of machine intelligence away from metaphysical essence and toward externally assessable performance (Turing, 1950). A similar orientation appears in the Dartmouth proposal, where learning and intelligence were treated as phenomena that might be specified with sufficient precision for machine simulation (McCarthy et al., 1955).

This early operational perspective matters because it changes the object of evaluation. Once intelligence is approached through observable operation, it becomes legitimate to ask not only whether a system succeeds at a single moment, but also whether its performance remains acceptable under repeated conditions of use. In that sense, repeated operation is not merely an extension of one-shot evaluation. It is a theoretically distinct problem.

This shift was reinforced by additional foundational work linking logic, learning, and adaptive organization. McCulloch and Pitts showed that nervous activity could be represented in logical form (McCulloch \& Pitts, 1943). Rosenblatt introduced a probabilistic and adaptive account of learning (Rosenblatt, 1958). Newell and Simon described computer science as an empirical inquiry into symbols and search (Newell \& Simon, 1976). Ashby emphasized adaptive organization and regulation rather than mere passive endurance (Ashby, 1960). Taken together, these traditions support a view of AI in which the relevant unit of analysis is not isolated success alone, but organized behavior unfolding across time.

Once repeated operation becomes the relevant object of study, a deeper theoretical problem emerges. A system may succeed on isolated tasks while gradually accumulating instability, unresolved load, or dependence on compensatory structure. It may also remain sustainable over long horizons not because it is free of burden, but because that burden stays bounded, damped, or structurally absorbed. The problem addressed in this paper is therefore not one-shot success, but long-run persistence: can an artificial system continue operating across indefinitely many cycles without incurring uncontrolled structural age?

\subsection{Why one-shot performance is insufficient}

One-shot performance is insufficient for this problem for at least three reasons. First, endpoint success suppresses temporal structure. A correct output at one cycle does not show whether the system is stabilizing, drifting, oscillating, or relying on increasingly fragile forms of compensation. Second, binary success-failure language suppresses degree. It does not distinguish between mild and severe deterioration, or between temporary disturbance and accumulating structural cost. Third, simple additive summaries may overcount effective load when several apparent weaknesses are partly overlapping manifestations of the same underlying instability.

These limitations become especially serious when the aim is to analyze repeated artificial operation rather than isolated task completion. A metric designed only for single outcomes cannot by itself reveal whether the burden associated with continued operation is diminishing, stabilizing, or compounding. For that reason, the present study does not begin from raw accuracy, event counts, or failure frequencies alone. Instead, it introduces a cycle-level function designed to track structural burden across repeated operations.

The point of introducing such a function is not merely to score how well the system performs at one moment. It is to formalize how burden evolves from one cycle to the next across an unbounded operational horizon. A suitable formal object for this purpose must satisfy at least three requirements. It must remain interpretable at the component level. It must remain mathematically controlled under repetition. And it must distinguish genuinely effective burden from burden that is only apparently multiple because of structural overlap.

The redundancy-adjusted Artificial Age Score is suitable precisely because it satisfies these requirements. It is componentwise decomposable. It is weighted rather than indiscriminate. It is logarithmically sensitive to deterioration. And it is explicitly adjusted for redundancy. In other words, it can represent not only how much weakness appears at a given cycle, but also how much of that weakness should count as effective structural load. This is the first main reason AAS belongs within a long-run persistence theory rather than being treated merely as another static metric.

\subsection{Information, reliability, and the rationale for AAS}

The mathematical shape of the framework is not arbitrary. One of its foundations lies in information theory. Shannon made logarithmic structure central to the quantitative treatment of uncertainty, transmission, and degradation (Shannon, 1948). Shannon is relevant here not because the Artificial Age Score is identical to entropy, but because his work provides a mathematical precedent for using logarithmic form in settings where sensitivity is nonlinear and where deterioration cannot be adequately represented by simple linear change alone.

That is the role played by the logarithmic penalty kernel in the present framework. A linear penalty would treat equal absolute losses as equivalent across all regions of the state space. The present model does not. It assigns greater structural significance to deterioration that occurs in already vulnerable regions. This makes the burden functional more appropriate for long-run persistence analysis, because repeated operation is often threatened less by isolated mild fluctuations than by losses concentrated in already weak parts of the system.

A second foundation lies in cybernetics and reliability theory. Wiener broadened the analysis of systems from communication alone to communication-and-control, thereby making feedback, adaptation, and stability central theoretical concerns (Wiener, 1948/1961). Von Neumann addressed the problem of how reliable systems can be constructed from unreliable components (von Neumann, 1956). This question is directly relevant to repeated AI operation, because the structural condition of an artificial system depends not only on local weakness, but also on how weakness is distributed, compensated, and prevented from compounding.

This is where the redundancy adjustment becomes necessary. A long-run persistence theory cannot rest on a quantity that is blind to overlap. If several observed weaknesses are partly manifestations of the same latent instability, then a purely additive measure can exaggerate effective burden. The redundancy-adjusted AAS is introduced precisely to avoid that problem. It treats age as a weighted functional of structural load rather than as a merely metaphorical label. It also discounts the portion of apparent burden that is redundant rather than genuinely independent. For a theory of repeated operation, this matters because the relevant question is not simply how many weaknesses appear, but how much effective structural load they impose.

This is the decisive reason AAS is used in the present paper. A theory of indefinite repetition cannot be built on a quantity that is unbounded by construction, insensitive to overlap, or meaningful only at a single instant. By contrast, the AAS-based cycle-level functional supplies a bounded quantity, preserves component structure through weighting, discounts apparent double counting through redundancy correction, and can be analyzed as a sequence over repeated cycles. It therefore provides the formal object needed to ask whether repeated operation remains merely ongoing, becomes cumulatively burdened, or asymptotically attenuates its own structural age.

\subsection{Disciplined repetition, infinity, and formal constraint}

The question of indefinite repetition requires conceptual discipline. The present study does not claim that artificial systems literally instantiate metaphysical infinity. Instead, it treats an unbounded cycle horizon as a legitimate object of mathematical analysis only when repeated operation is governed by explicit structural constraints. For that reason, the framework is concerned not with endlessness as such, but with bounded cycle-level age, convergence, cumulative burden, and asymptotic regime.

This perspective has a clear mathematical background. Bolzano helped establish that infinite or indefinitely extensible structures can be treated as genuine objects of analysis rather than as merely rhetorical notions (Bolzano, 1851). Cantor gave the treatment of infinite order a stronger formal footing (Cantor, 1874). Dedekind reinforced that structural orientation through his treatment of number systems and infinite sets (Dedekind, 1888). These developments matter here because the present study is concerned not with metaphysical endlessness, but with formally governed non-terminal repetition.

At the same time, the history of foundations also revealed the dangers of unrestricted totalization. Russell’s paradox showed that unconstrained set formation can generate contradiction (Russell, 1903). Zermelo’s axiomatization made clear why formal restriction is necessary (Zermelo, 1908). Hilbert later emphasized that infinity becomes mathematically tractable only under disciplined symbolic control (Hilbert, 1926). The relevant lesson for the present study is straightforward: indefinite repetition becomes mathematically meaningful only when the repeated object is explicitly defined, bounded where necessary, and stable enough to support asymptotic reasoning.

Frege is relevant in a precise and limited sense. In \textit{Begriffsschrift}, he made formal definition and exact logical articulation central to the treatment of mathematical reasoning (Frege, 1879/1967). In \textit{The Foundations of Arithmetic}, he argued that number should not be understood as a vague empirical accumulation, but as something requiring exact conceptual determination (Frege, 1884/1980). Contemporary scholarship likewise emphasizes that Frege’s project turned logical structure into a rigorous instrument for the reconstruction of arithmetic (Cook, 2023). That orientation matters here because the present study also transforms repeated cycles into a formally structured sequence rather than leaving them at the level of informal temporal description.

Kant is relevant for a different but complementary reason. In the \textit{Critique of Pure Reason}, mathematical cognition is tied to determinate construction rather than mere empirical succession (Kant, 1781/1787). Contemporary scholarship on Kant’s philosophy of mathematics has emphasized exactly this connection between construction and mathematical determinacy (Shabel, 2013). That point matters here because repeated cycles do not become mathematically significant simply by continuing. They become analyzable only when the governing burden functional is explicitly structured and rule bound.

Aristotle may also be used in a narrow and controlled way. His distinction between potentiality and actuality helps clarify that the mere possibility of indefinite continuation is not yet the same as a determinate operational state (Aristotle, trans.\ 1998). Contemporary reference treatments of Aristotle’s metaphysics continue to treat this distinction as central to his account of realized form and determinate being (Cohen, 2000). In that limited sense, Aristotle helps distinguish mere repeatability from a specific realized mode of long-run operation, whether as bounded persistence, burdened persistence, oscillatory persistence, or zero-burden persistence.

Kuhn provides a further contextual support, but only in a qualified sense. In \textit{The Structure of Scientific Revolutions}, he showed that scientific development involves shifts in what counts as a legitimate problem, method, and standard of assessment (Kuhn, 1962/2012). Contemporary discussions of Kuhn likewise emphasize changes in conceptual and methodological frames rather than the mere addition of isolated techniques (Bird, 2004). In that narrower sense, the move from one-shot benchmarking to longitudinal structural evaluation may be understood as a shift in evaluative perspective rather than as the simple introduction of another metric.

A final caution comes from Gödel. His incompleteness results showed that sufficiently expressive formal systems face intrinsic limits of derivability (Gödel, 1931). Gödel is relevant here not because the present framework is thereby invalidated, but because it reminds us that no formal model should be mistaken for an exhaustive theory of intelligence, semantics, or cognition. The aim of this study is therefore narrower and more rigorous: to formalize one specific aspect of repeated artificial operation, namely the long-run behavior of structural burden.

\subsection{Research gap, position of the study, and contributions}

Despite major traditions in AI, information theory, cybernetics, reliability theory, and formal philosophy, a gap remains between static evaluation and a rigorous account of long-run structural persistence. Existing traditions explain why performance, uncertainty, feedback, compensation, and formal discipline matter. What they do not by themselves provide is a cycle-level mathematical framework capable of distinguishing bounded age, burdened persistence, zero-burden persistence, oscillatory persistence, and cumulative terminal burden under repeated artificial operation.

The present study addresses that gap by extending the redundancy-adjusted Artificial Age Score into a long-run persistence framework. Its aim is not to redefine intelligence in general, nor to anthropomorphize AI through a loose metaphor of aging. Rather, it introduces a cycle-level formalism in which structural age is represented as a bounded, weighted, and redundancy-aware burden process. This makes it possible to ask a sharper question than whether a system merely continues to operate. Under repeated use, does its burden stabilize at a positive level? Does it remain bounded while fluctuating? Does it accumulate without bound in the aggregate? Or does it attenuate toward zero?

The study makes five main contributions. First, it reformulates the redundancy-adjusted Artificial Age Score as a cycle-level functional that generates an age sequence across repeated operation. Second, it defines asymptotic operating regimes, including burdened persistence, zero-burden persistence, oscillatory persistence, and cumulative terminal burden. Third, it establishes the analytical properties required for a persistence theory, including well-definedness, boundedness, comparative ordering, perturbation sensitivity, and convergence behavior. Fourth, it shows that, within the proposed framework, indefinite cyclic continuation does not require unbounded structural aging. Fifth, it provides a formal language for distinguishing systems that merely continue operating from systems whose structural burden is asymptotically attenuated.

The main research question is therefore: Can an AI system persist indefinitely without unbounded structural aging?

The more specific theoretical sub-question is: Under what mathematical conditions can repeated artificial operation remain compatible with bounded cycle-level age, vanishing marginal aging, and, in the strongest regime, asymptotically vanishing structural burden?
\section{Formal Framework}

\subsection{Cycle-level age as a formal object}

Following the redundancy-adjusted Artificial Age Score framework introduced by Kayadibi (2026), repeated artificial operation is modeled over discrete cycles indexed by \(n \in \mathbb{N}\). At each cycle, the system is described through a finite collection of components \(i = 1, \dots, m\), where \(m < \infty\). These components may represent modules, subsystems, thematic dimensions, or any other analytically distinguished units of artificial operation. The finiteness of \(m\) ensures that the cycle-level age functional is defined through a finite weighted aggregation rather than through an infinite series.

The cycle-level age of the system at cycle \(n\) is defined by
\begin{equation}
A_n = \sum_{i=1}^{m} w_i (1 - R_{n,i}) \phi(x_{n,i}).
\end{equation}

In this expression, \(x_{n,i}\) denotes the consistency or robustness level of component \(i\) at cycle \(n\), \(R_{n,i}\) denotes the redundancy correction term, \(w_i\) denotes the component weight, and \(A_n\) denotes the realized structural age of the system at cycle \(n\).

This definition is the formal core of the model. It states that the age of the system at a given cycle is not treated as a metaphorical label, but as a weighted structural burden obtained by aggregating component-level penalties. Each component contributes according to three factors: its weight in the system, its non-redundant burden, and the penalty associated with its current consistency level. In this way, age is represented as a cycle-specific mathematical quantity rather than as a vague descriptive notion.

The sequence
\[
\{A_n\}_{n=1}^{\infty}
\]
therefore becomes the central object of long-run analysis. The theoretical problem of persistence is then reformulated as a question about the asymptotic behavior of this age sequence under repeated operation.

\subsection{The logarithmic penalty kernel}

Following Kayadibi (2026), let \(\varepsilon > 0\) be a fixed regularization parameter. The penalty kernel is defined by
\begin{equation}
\phi(x) = -\log_2\!\left(\frac{x+\varepsilon}{1+\varepsilon}\right)
= \log_2\!\left(\frac{1+\varepsilon}{x+\varepsilon}\right),
\qquad x \in (0,1].
\end{equation}

The parameter \(\varepsilon\) serves as a regularization term. Its role is to ensure that the penalty remains finite and well defined throughout the domain \((0,1]\), especially near very small consistency values. Without this term, the logarithmic expression would become singular as \(x \to 0\). By introducing \(\varepsilon\), the framework preserves logarithmic nonlinearity while preventing divergence of the kernel near the lower edge of the admissible domain.

The interpretation of \(\phi\) is straightforward but important. A higher consistency level \(x\) produces a smaller penalty, whereas a lower consistency level produces a larger penalty. Because the kernel is logarithmic, deterioration is not penalized linearly. Instead, the same absolute loss in consistency produces a stronger penalty when the system is already in a weak state than when it remains in a strong one. This makes the model more sensitive to degradation in vulnerable regions.

The logarithmic form is therefore not a decorative choice. It provides the mathematical mechanism through which nonlinear structural sensitivity is built into the age functional.
\subsection{Domain restrictions, weights, severity channel, and redundancy correction}

Consistent with the redundancy-adjusted AAS framework in Kayadibi (2026), the following domain restrictions are imposed:
\begin{equation}
x_{n,i}\in(0,1], \qquad R_{n,i}\in[0,1], \qquad w_i\ge 0, \qquad \sum_{i=1}^{m} w_i=1.
\end{equation}

These assumptions define the admissible state space of the model.

First, the condition \(x_{n,i}\in(0,1]\) means that each component consistency level is strictly positive and at most one. The upper value \(x_{n,i}=1\) corresponds to perfect consistency, while smaller values indicate increasing structural weakness. The exclusion of zero ensures that the logarithmic penalty remains well defined on the stated domain.

Second, the condition \(R_{n,i}\in[0,1]\) means that redundancy is represented as a bounded correction factor. If \(R_{n,i}=0\), then no redundancy correction is applied and the full penalty of the component is counted. If \(R_{n,i}\) is closer to one, then a larger share of the apparent burden is treated as structurally redundant. Thus, the factor \(1-R_{n,i}\) represents the effective non-redundant share of the component’s burden.

Third, the weights \(w_i\) satisfy \(w_i\ge 0\) and \(\sum_{i=1}^{m} w_i=1\). This normalization ensures that the age functional is interpreted on a common weighted scale rather than as an unnormalized sum whose magnitude depends directly on the number of components. A component with larger weight contributes more strongly to cycle-level age, reflecting greater structural importance.

A component \(i\) is called active if \(w_i>0\). This distinction matters because only active components influence the value of \(A_n\). Components with zero weight may remain part of the formal indexing set, but they do not affect the realized age of the system.

In the present framework, severity and redundancy are analytically distinct. The term \(R_{n,i}\) does not measure the magnitude of weakness itself. Rather, it measures the extent to which the burden attributed to component \(i\) overlaps with that of other burden-bearing structures, so that the effective contribution of the component is reduced by the factor \(1-R_{n,i}\). Here \(V_{n,i}\) denotes an underlying component-level violation magnitude, whereas \(\Delta_n\), introduced later among the inter-cycle quantities, denotes the positive cycle-to-cycle increment in system-level age.

By contrast, the severity channel of the model is carried by \(x_{n,i}\) and \(\phi(x_{n,i})\). More generally, whenever the consistency score \(x_{n,i}\) is derived from an underlying violation magnitude \(V_{n,i}\), the intended logic is
\begin{equation}
V_{n,i}\uparrow \;\Rightarrow\; x_{n,i}\downarrow \;\Rightarrow\; \phi(x_{n,i})\uparrow.
\end{equation}

That is, larger violations correspond to lower consistency, lower consistency yields a larger penalty, and a larger penalty increases the raw burden contribution before redundancy adjustment is applied. Redundancy then enters separately through
\begin{equation}
R_{n,i}\uparrow \;\Rightarrow\; (1-R_{n,i})\downarrow,
\end{equation}
so that greater structural overlap reduces the effective independent contribution of that penalty to cycle-level age.

The model therefore has a two-channel structure. The term \(\phi(x_{n,i})\) captures the severity-sensitive burden attached to the component’s current consistency level, whereas the factor \((1-R_{n,i})\) discounts the portion of that burden that is structurally redundant rather than genuinely independent. Accordingly, the contribution of component \(i\) to cycle-level age is not \(\phi(x_{n,i})\) alone, but
\begin{equation}
w_i(1-R_{n,i})\phi(x_{n,i}).
\end{equation}

Taken together, these restrictions and distinctions ensure that the cycle-level age functional is mathematically well defined and structurally interpretable. The framework does not merely count weaknesses; it weights them, discounts structurally overlapping burden, and penalizes them nonlinearly through the kernel \(\phi\).
\subsection{Inter-cycle quantities}

To study how structural age evolves across repeated operation, three inter-cycle quantities are introduced:
\begin{equation}
\delta_n := A_{n+1}-A_n,
\end{equation}
\begin{equation}
D_n := |A_{n+1}-A_n|,
\end{equation}
\begin{equation}
\Delta_n := \max\{0, A_{n+1}-A_n\}.
\end{equation}

These quantities separate three analytically distinct aspects of cycle-to-cycle change.

The first quantity, \(\delta_n\), is the signed increment. It records both the direction and the magnitude of the change between two consecutive cycles. If \(\delta_n>0\), age has increased. If \(\delta_n<0\), age has decreased. If \(\delta_n=0\), age is unchanged from one cycle to the next.

The second quantity, \(D_n\), is the absolute inter-cycle variation. Unlike \(\delta_n\), it ignores direction and measures only the size of the fluctuation. It therefore captures instability in the age process regardless of whether that instability takes the form of increase or decrease.

The third quantity, \(\Delta_n\), isolates positive age accumulation. It records only upward movement and assigns negative movement a value of zero. This is useful because not every fluctuation corresponds to new structural burden. A system may oscillate considerably without exhibiting persistent accumulation of age. The quantity \(\Delta_n\) is introduced precisely to distinguish general fluctuation from one-sided worsening.

This separation is essential for long-run analysis. A theory of persistence should not confuse variability, one-sided accumulation, and stabilization, since these correspond to mathematically distinct behaviors of the age process.

\subsection{Cumulative burden}

In addition to cycle-level age, the framework also considers the cumulative burden up to cycle \(N\):
\begin{equation}
\mathcal{C}_N := \sum_{n=1}^{N} A_n.
\end{equation}

This quantity represents the total structural load accumulated across the first \(N\) cycles. Its introduction is necessary because bounded cycle-level age does not by itself imply bounded accumulated burden over time.

This distinction is one of the central conceptual features of the framework. The quantity \(A_n\) describes the local structural condition of the system at cycle \(n\). By contrast, \(\mathcal{C}_N\) describes the total burden carried across a growing operational horizon. A system may therefore remain age-bounded at each individual cycle while still accumulating an unbounded total burden when repeated often enough.

For this reason, long-run persistence cannot be evaluated solely at the level of \(A_n\). Both local and cumulative behavior must be studied. The former concerns pointwise structural control, whereas the latter concerns the total cost of indefinite repetition.
\subsection{Asymptotic operating regimes}

The framework distinguishes several asymptotic regimes of repeated artificial operation.

\begin{definition}[Burdened persistent regime]
A system is said to be in a burdened persistent regime if there exists \(A^\ast>0\) such that
\begin{equation}
A_n\to A^\ast.
\end{equation}
\end{definition}

In this regime, the age sequence stabilizes, but it stabilizes at a strictly positive level. The system therefore remains operational under a persistent non-vanishing cycle-level burden. This corresponds to long-run continuation with structural cost.

\begin{definition}[Zero-burden persistent regime]
A system is said to be in a zero-burden persistent regime if
\begin{equation}
A_n\to 0.
\end{equation}
\end{definition}

This is the strongest non-terminal regime in the present framework. Here the system not only avoids explosive age growth, but asymptotically suppresses its cycle-level structural burden. Repeated operation continues, yet the age attached to each cycle tends to disappear.

\begin{definition}[Oscillatory persistent regime]
A system is said to be in an oscillatory persistent regime if the sequence \(\{A_n\}\) remains bounded but does not converge.
\end{definition}

This regime represents persistence without asymptotic settling. The system does not exhibit explosive aging, but neither does it approach a single limiting burden level. Instead, its cycle-level age continues to fluctuate within a bounded range.

\begin{definition}[Cumulative terminal burden]
A system is said to exhibit cumulative terminal burden if
\begin{equation}
\mathcal{C}_N\to\infty.
\end{equation}
\end{definition}

This notion is deliberately defined at the cumulative level rather than at the pointwise level. Under the present model, the possibility of explosive cycle-age growth is excluded once boundedness is established. Nevertheless, cumulative divergence may still occur. A system may therefore remain locally age-bounded while carrying an unbounded total burden over an infinite horizon.

These regimes clarify the sense in which long-run persistence is understood in the present study. The relevant question is not simply whether the system continues to operate, but whether its age stabilizes at a positive level, vanishes, oscillates without convergence, or accumulates without bound in the aggregate. The theory of persistence is thus organized not around a binary contrast between success and failure, but around a hierarchy of asymptotic burden profiles.
\section{Basic Analytical Properties}

\subsection{Positivity and uniform upper bound of the kernel}

The first step is to establish that the penalty kernel is always nonnegative and remains uniformly bounded on its admissible domain.

\begin{lemma}[Positivity and uniform upper bound]
For every \(x\in(0,1]\),
\begin{equation}
0\le \phi(x)<B_\varepsilon,
\end{equation}
where
\begin{equation}
B_\varepsilon:=\log_2\!\left(\frac{1+\varepsilon}{\varepsilon}\right).
\end{equation}
\end{lemma}

\begin{proof}
Since
\[
\phi(x)=\log_2\!\left(\frac{1+\varepsilon}{x+\varepsilon}\right),
\]
and since \(x\le 1\), one has
\[
x+\varepsilon\le 1+\varepsilon.
\]
Hence,
\[
\frac{1+\varepsilon}{x+\varepsilon}\ge 1,
\]
and therefore
\[
\phi(x)\ge 0.
\]

On the other hand, because \(x>0\), it follows that
\[
x+\varepsilon>\varepsilon.
\]
Therefore,
\[
\frac{1+\varepsilon}{x+\varepsilon}<\frac{1+\varepsilon}{\varepsilon}.
\]
Since the logarithm is increasing, it follows that
\[
\phi(x)<\log_2\!\left(\frac{1+\varepsilon}{\varepsilon}\right)=B_\varepsilon.
\]
\end{proof}

This lemma establishes two basic facts. First, the kernel never assigns negative burden. Second, it cannot become arbitrarily large on the admissible domain. These properties are the foundation for the boundedness of the cycle-level age functional.
\subsection{Monotonicity and convexity of the kernel}

The next result describes how the kernel responds to changes in component consistency.

\begin{lemma}[Monotonicity and convexity]
The kernel \(\phi\) is strictly decreasing and strictly convex on \((0,1]\). In particular,
\begin{equation}
\phi'(x)=-\frac{1}{(x+\varepsilon)\ln 2}<0,
\end{equation}
and
\begin{equation}
\phi''(x)=\frac{1}{(x+\varepsilon)^2\ln 2}>0.
\end{equation}
\end{lemma}

\begin{proof}
The kernel may be written as
\[
\phi(x)=\log_2(1+\varepsilon)-\log_2(x+\varepsilon).
\]
Differentiation gives
\[
\phi'(x)=-\frac{1}{(x+\varepsilon)\ln 2},
\]
which is strictly negative for all \(x\in(0,1]\). Differentiating once more yields
\[
\phi''(x)=\frac{1}{(x+\varepsilon)^2\ln 2}>0.
\]
Thus, \(\phi\) is strictly decreasing and strictly convex on \((0,1]\).
\end{proof}

This result provides the local geometry of the penalty function. A higher consistency level produces a smaller age penalty, and the convexity means that deterioration is penalized more sharply in already weak regions. The framework is therefore designed to be more sensitive where vulnerability is already high.

This point is fully consistent with the severity logic introduced in 2.3. There, larger underlying violation magnitudes were taken to reduce consistency. Since \(\phi\) is strictly decreasing in \(x\), any increase in violation severity that lowers consistency necessarily raises the penalty term. Thus, the chain
\begin{equation}
V_{n,i}\uparrow \;\Rightarrow\; x_{n,i}\downarrow \;\Rightarrow\; \phi(x_{n,i})\uparrow
\end{equation}
follows from the combination of the consistency transformation and the monotonicity of \(\phi\). In this sense, Lemma 2 provides the analytical foundation for the severity channel of the model.
\subsection{Well-definedness and boundedness of cycle-level age}

The previous lemmas immediately imply that the cycle-level age is mathematically well defined.

\begin{theorem}[Well-definedness and uniform boundedness]
For every cycle \(n\),
\begin{equation}
0\le A_n<B_\varepsilon.
\end{equation}
In particular, the cycle-level age is well defined and uniformly bounded.
\end{theorem}

\begin{proof}
For each component \(i\), the following conditions hold:
\[
w_i\ge 0,\qquad 1-R_{n,i}\in[0,1],\qquad \phi(x_{n,i})\ge 0,
\]
where the last inequality follows from Lemma 1. Hence,
\[
w_i(1-R_{n,i})\phi(x_{n,i})\ge 0.
\]
Summing over all \(i\) gives
\[
A_n\ge 0.
\]

Again by Lemma 1,
\[
\phi(x_{n,i})<B_\varepsilon
\]
for every component. Since \(1-R_{n,i}\le 1\), one obtains
\[
w_i(1-R_{n,i})\phi(x_{n,i})<w_iB_\varepsilon.
\]
Summing over all components yields
\[
A_n<B_\varepsilon\sum_{i=1}^{m}w_i=B_\varepsilon.
\]
Therefore,
\[
0\le A_n<B_\varepsilon.
\]
\end{proof}

This theorem is one of the central structural results of the framework. It shows that the age assigned to any individual cycle is always finite, nonnegative, and uniformly controlled. The cycle index may extend indefinitely, but the pointwise age value remains confined to a fixed bounded interval.

The theorem also clarifies the role of redundancy correction in the boundedness argument. The factor \((1-R_{n,i})\) can only reduce, and never enlarge, the raw penalty contribution \(\phi(x_{n,i})\), since \(R_{n,i}\in[0,1]\). Thus, redundancy adjustment preserves the boundedness of the age functional while discounting the effective independent burden attributed to each component.
\subsection{Immediate consequence for the age sequence}

Theorem 1 implies an immediate global property of the sequence \(\{A_n\}_{n=1}^{\infty}\).

\begin{corollary}[Uniform boundedness of the age sequence]
The sequence \(\{A_n\}_{n=1}^{\infty}\) is uniformly bounded in \([0,B_\varepsilon)\).
\end{corollary}

\begin{proof}
This follows directly from Theorem 1, since the stated bound holds for every cycle \(n\).
\end{proof}

This corollary may appear elementary, but it has important theoretical consequences. It rules out explosive pointwise aging at the level of individual cycles and makes it meaningful to study long-run persistence over an unbounded horizon without leaving a bounded age domain.

\subsection{Why these properties matter for the theory}

Taken together, Lemmas 1 and 2 and Theorem 1 establish the basic mathematical stability of the proposed age functional. The kernel is nonnegative, bounded, decreasing, and convex, while the cycle-level age is well defined and uniformly bounded for all cycles. These results ensure that the subsequent persistence analysis is built on a controlled formal object rather than on an unstable or ill-posed score.

In conceptual terms, this means that the framework is strong enough to support long-run asymptotic analysis. The severity channel is well behaved because larger violations, once transformed into lower consistency values, necessarily induce larger penalties through the monotonic structure of \(\phi\). The redundancy channel is also well behaved because \(R_{n,i}\in[0,1]\) can only reduce the effective contribution of a component-level penalty rather than destabilize it. Since the cycle-level age cannot become negative and cannot explode pointwise, the relevant theoretical questions shift from mere definability to the more substantive issues of ordering, perturbation sensitivity, stabilization, oscillation, and cumulative burden. The next section develops these properties.

\section{Ordering and Stability Theory}

\subsection{Comparative age ordering}

A useful age functional should do more than assign bounded values. It should also induce a meaningful ordering across different system states. The next result shows that the proposed age functional respects the intuitive comparison that lower consistency and lower redundancy correction should not produce a smaller age.

\begin{theorem}[Comparative age ordering]
Let
\begin{equation}
A(x,R)=\sum_{i=1}^{m} w_i(1-R_i)\phi(x_i), \qquad A(y,S)=\sum_{i=1}^{m} w_i(1-S_i)\phi(y_i).
\end{equation}
If, for every \(i\),
\[
x_i\le y_i \qquad \text{and} \qquad R_i\le S_i,
\]
then
\begin{equation}
A(x,R)\ge A(y,S).
\end{equation}
\end{theorem}

\begin{proof}
Since \(x_i\le y_i\) and the kernel \(\phi\) is strictly decreasing by Lemma 2, it follows that
\[
\phi(x_i)\ge \phi(y_i).
\]
Likewise, from \(R_i\le S_i\), one obtains
\[
1-R_i\ge 1-S_i.
\]
Because these factors are nonnegative, it follows that
\[
(1-R_i)\phi(x_i)\ge (1-S_i)\phi(y_i)
\]
for every \(i\). Multiplying by \(w_i\ge 0\) and summing over all components yields
\[
A(x,R)\ge A(y,S).
\]
\end{proof}

This theorem shows that the age functional is order-consistent. If one system is nowhere better in consistency and nowhere better in redundancy correction, then it cannot be assigned a smaller structural age. In terms of the two-channel interpretation introduced earlier, the theorem states that deterioration in the severity channel and weakening in the redundancy channel both move the effective age contribution in the same direction: they do not reduce structural burden, but increase or preserve it. In this sense, the functional behaves as a coherent comparative measure rather than as a purely numerical aggregation without interpretive order.
\subsection{Sensitivity with respect to consistency and redundancy perturbations}

For long-run analysis, it is not enough to know that the age functional is ordered. One must also know whether small changes in internal variables produce controlled changes in age. The next theorem establishes such a perturbation bound.

\begin{theorem}[Sensitivity bound]
Fix \(\alpha>0\), and suppose that
\[
x_i,y_i\in[\alpha,1] \qquad \text{for all } i.
\]
Then
\begin{equation}
|A(x,R)-A(y,S)| \le L_\alpha \sum_{i=1}^{m} w_i|x_i-y_i| + B_\varepsilon \sum_{i=1}^{m} w_i|R_i-S_i|,
\end{equation}
where
\begin{equation}
L_\alpha:=\frac{1}{(\alpha+\varepsilon)\ln 2}.
\end{equation}
\end{theorem}

\begin{proof}
Write
\[
A(x,R)-A(y,S) = \sum_{i=1}^{m} w_i\Big[(1-R_i)\phi(x_i)-(1-S_i)\phi(y_i)\Big].
\]
For each component \(i\), one has the identity
\[
(1-R_i)\phi(x_i)-(1-S_i)\phi(y_i)
=
(1-R_i)\big(\phi(x_i)-\phi(y_i)\big) + (S_i-R_i)\phi(y_i).
\]
Taking absolute values gives
\[
\left|(1-R_i)\phi(x_i)-(1-S_i)\phi(y_i)\right|
\le
|1-R_i|\,|\phi(x_i)-\phi(y_i)| + |S_i-R_i|\,|\phi(y_i)|.
\]

Since \(R_i\in[0,1]\), one has
\[
|1-R_i|\le 1.
\]
Also, by Lemma 1,
\[
|\phi(y_i)|\le B_\varepsilon.
\]

Next, by the Mean Value Theorem and Lemma 2, the derivative of \(\phi\) satisfies
\[
|\phi'(x)|\le \frac{1}{(\alpha+\varepsilon)\ln 2}=L_\alpha \qquad \text{for all } x\in[\alpha,1].
\]
Hence,
\[
|\phi(x_i)-\phi(y_i)| \le L_\alpha |x_i-y_i|.
\]
Therefore,
\[
\left|(1-R_i)\phi(x_i)-(1-S_i)\phi(y_i)\right|
\le
L_\alpha |x_i-y_i| + B_\varepsilon |R_i-S_i|.
\]
Multiplying by \(w_i\) and summing over all \(i\) yields
\[
|A(x,R)-A(y,S)| \le L_\alpha \sum_{i=1}^{m} w_i|x_i-y_i| + B_\varepsilon \sum_{i=1}^{m} w_i|R_i-S_i|.
\]
\end{proof}

This theorem gives a weighted Lipschitz-type stability estimate. It shows that the age functional does not react in an uncontrolled way to small perturbations in component consistency and redundancy variables, provided consistency remains bounded away from zero by \(\alpha\). In the language of the two-channel model, the theorem shows that neither the severity channel nor the redundancy channel can generate arbitrarily large jumps in structural age from arbitrarily small perturbations. This is important because a long-run persistence theory would be of limited use if small internal changes could produce unbounded discontinuities in the effective burden process.
\subsection{Continuity and componentwise convergence}

The previous theorem controls finite perturbations between two states. The next result addresses repeated operation directly by showing that stabilization of the internal variables implies stabilization of the age sequence.

\begin{theorem}[Componentwise convergence implies age convergence]
If, for every \(i\),
\[
x_{n,i}\to x_i^\ast\in(0,1], \qquad R_{n,i}\to R_i^\ast\in[0,1],
\]
then
\begin{equation}
A_n\to A^\ast,
\end{equation}
where
\begin{equation}
A^\ast=\sum_{i=1}^{m} w_i(1-R_i^\ast)\phi(x_i^\ast).
\end{equation}
\end{theorem}

\begin{proof}
Since the kernel \(\phi\) is continuous on \((0,1]\), the convergence
\[
x_{n,i}\to x_i^\ast
\]
implies
\[
\phi(x_{n,i})\to \phi(x_i^\ast).
\]
Likewise, from
\[
R_{n,i}\to R_i^\ast,
\]
it follows that
\[
1-R_{n,i}\to 1-R_i^\ast.
\]
By continuity of multiplication,
\[
(1-R_{n,i})\phi(x_{n,i}) \to (1-R_i^\ast)\phi(x_i^\ast)
\]
for each component \(i\). Since the number of components is finite, summation preserves the limit, and therefore
\[
A_n = \sum_{i=1}^{m} w_i(1-R_{n,i})\phi(x_{n,i}) \to \sum_{i=1}^{m} w_i(1-R_i^\ast)\phi(x_i^\ast) = A^\ast.
\]
\end{proof}

This theorem connects internal stabilization to observable structural stabilization. If each active component settles to a limiting consistency level and a limiting redundancy correction, then the cycle-level age also settles to a limiting value. In terms of the two-channel interpretation, stabilization in both the severity channel and the redundancy channel induces stabilization of the aggregated age process. The result is especially important because it turns the age sequence into a genuine summary of internal long-run structure rather than merely a detached external score.

\subsection{Stability interpretation of the age functional}

Taken together, Theorems 2 through 4 show that the proposed age functional has three basic stability properties.

First, it respects comparative order: weaker consistency and weaker redundancy correction do not lead to smaller age. Second, it is perturbation-stable: bounded changes in internal variables lead to controlled changes in age. Third, it is continuous under repeated operation: if the component-level variables converge, then the cycle-level age converges as well.

These properties matter because they justify the use of the functional as the basis of a long-run persistence theory. Without comparative order, the score would lack structural meaning. Without perturbation control, it would be too unstable for longitudinal analysis. Without continuity under repeated operation, it would fail to reflect internal stabilization in a coherent way.

More specifically, these properties show that the two channels of the model are analytically coordinated rather than arbitrary. The severity channel, carried by \(\phi(x_{n,i})\), reacts in an ordered and controlled manner to changes in consistency. The redundancy channel, carried by \((1-R_{n,i})\), discounts effective burden without destroying boundedness, order, or continuity. The cycle-level age therefore remains interpretable as an aggregated effective burden process rather than as an unstable numerical construction. The next section builds on these properties to establish the asymptotic persistence results of the framework.
\section{Long-Run Persistence Theory}

The previous sections established that the proposed age functional is well defined, uniformly bounded, order-consistent, perturbation-stable, and continuous under componentwise convergence. The next step is to determine what these properties imply for repeated operation over an unbounded cycle horizon. The present section develops the central long-run results of the framework.

\subsection{Finite-variation persistence}

A natural way to express long-run stabilization is to require that the total amount of cycle-to-cycle fluctuation be finite. The next theorem shows that this condition is already strong enough to force convergence of the age sequence.

\begin{theorem}[Finite-variation persistence]
If
\begin{equation}
\sum_{n=1}^{\infty}|A_{n+1}-A_n|<\infty,
\end{equation}
then the sequence \(\{A_n\}\) converges to some \(A^\ast\in[0,B_\varepsilon]\). Moreover,
\begin{equation}
D_n\to 0 \qquad \text{and} \qquad \Delta_n\to 0.
\end{equation}
\end{theorem}

\begin{proof}
For any integers \(q>p\), one has the telescoping identity
\[
A_q-A_p=\sum_{n=p}^{q-1}(A_{n+1}-A_n).
\]
Taking absolute values yields
\[
|A_q-A_p| \le \sum_{n=p}^{q-1}|A_{n+1}-A_n|.
\]
Since the series
\[
\sum_{n=1}^{\infty}|A_{n+1}-A_n|
\]
converges, its tails converge to zero. Hence \(\{A_n\}\) is a Cauchy sequence. Because \(\mathbb{R}\) is complete, there exists \(A^\ast\in\mathbb{R}\) such that
\[
A_n\to A^\ast.
\]

By Theorem 1,
\[
0\le A_n<B_\varepsilon \qquad \text{for all } n.
\]
Therefore, the limit must satisfy
\[
A^\ast\in[0,B_\varepsilon].
\]

Now recall that
\[
D_n=|A_{n+1}-A_n|.
\]
Since the series \(\sum_{n=1}^{\infty}D_n\) converges and all terms are nonnegative, it follows that
\[
D_n\to 0.
\]
Finally, because
\[
0\le \Delta_n\le D_n,
\]
the squeeze argument gives
\[
\Delta_n\to 0.
\]
\end{proof}

This theorem is fundamental because it shows that full componentwise convergence is not required in order to obtain long-run stabilization. It is enough that the total amount of inter-cycle variation be finite. In that case, the age process converges, total fluctuation eventually disappears, and positive marginal age accumulation vanishes. Since the age functional aggregates both the severity channel and the redundancy-discounted effective burden channel, the theorem shows that long-run stabilization can be obtained at the level of the total effective structural age even without separately imposing convergence of every underlying component process.
\subsection{Geometric stabilization}

The finite-variation condition of Theorem 5 is abstract but powerful. The next theorem gives a concrete sufficient condition for it by assuming that componentwise perturbations decay geometrically over time.

\begin{theorem}[Geometric stabilization]
Let \(\alpha>0\) and \(q\in(0,1)\). Suppose that for every \(i\), there exist constants \(a_i,b_i\ge 0\) such that, for all \(n\),
\[
x_{n,i}\ge \alpha,
\]
\[
|x_{n+1,i}-x_{n,i}|\le a_i q^n,
\]
and
\[
|R_{n+1,i}-R_{n,i}|\le b_i q^n.
\]
Then \(A_n\) converges. More precisely,
\begin{equation}
D_n\le Cq^n,
\end{equation}
where
\begin{equation}
C:=L_\alpha\sum_{i=1}^{m}w_i a_i + B_\varepsilon\sum_{i=1}^{m}w_i b_i.
\end{equation}
Consequently,
\begin{equation}
\sum_{n=1}^{\infty}D_n<\infty, \qquad A_n\to A^\ast, \qquad \Delta_n\to 0.
\end{equation}
\end{theorem}

\begin{proof}
Apply Theorem 3 to two consecutive cycles \(n\) and \(n+1\). Then
\[
|A_{n+1}-A_n| \le L_\alpha\sum_{i=1}^{m}w_i|x_{n+1,i}-x_{n,i}| + B_\varepsilon\sum_{i=1}^{m}w_i|R_{n+1,i}-R_{n,i}|.
\]
By the stated assumptions,
\[
|A_{n+1}-A_n| \le L_\alpha\sum_{i=1}^{m}w_i a_i q^n + B_\varepsilon\sum_{i=1}^{m}w_i b_i q^n.
\]
Factoring out \(q^n\) gives
\[
D_n=|A_{n+1}-A_n| \le q^n\left( L_\alpha\sum_{i=1}^{m}w_i a_i + B_\varepsilon\sum_{i=1}^{m}w_i b_i \right)=Cq^n.
\]
Since the geometric series converges,
\[
\sum_{n=1}^{\infty}D_n \le C\sum_{n=1}^{\infty}q^n < \infty.
\]
The conclusion therefore follows from Theorem 5.
\end{proof}

This theorem gives a more operational criterion for stabilization. If consistency variables in the severity channel and redundancy corrections in the overlap-discount channel both change less and less over time at a geometric rate, then the aggregate age process inherits that damping. The result is important because it links a concrete decay law at the component level to a provable convergence law at the level of total structural age.
\subsection{Zero-burden necessity and cumulative divergence under positive limiting burden}

The next result separates two different questions: whether the cycle-level burden tends to zero, and whether the cumulative burden remains finite. These are related, but they are not identical.

\begin{theorem}[Zero-burden necessity]
If
\begin{equation}
\sum_{n=1}^{\infty}A_n<\infty,
\end{equation}
then necessarily
\begin{equation}
A_n\to 0.
\end{equation}
\end{theorem}

\begin{proof}
By Theorem 1,
\[
A_n\ge 0 \qquad \text{for all } n.
\]
Thus, \(\sum_{n=1}^{\infty}A_n\) is a series with nonnegative terms. A necessary condition for convergence of such a series is that its general term tends to zero. Hence,
\[
A_n\to 0.
\]
\end{proof}

This theorem states that finite total accumulated burden is possible only if the cycle-level burden itself eventually vanishes. In other words, if a system keeps paying a non-negligible structural cost per cycle, then that cost cannot remain summable forever.

\begin{corollary}[Positive persistent burden implies divergent cumulative load]
If
\begin{equation}
A_n\to A^\ast>0,
\end{equation}
then
\begin{equation}
\sum_{n=1}^{\infty}A_n=\infty.
\end{equation}
\end{corollary}

\begin{proof}
Since \(A_n\to A^\ast>0\), there exists \(N_0\) such that for all \(n\ge N_0\),
\[
A_n\ge \frac{A^\ast}{2}>0.
\]
Hence,
\[
\sum_{n=N_0}^{\infty}A_n \ge \sum_{n=N_0}^{\infty}\frac{A^\ast}{2}=\infty.
\]
Therefore,
\[
\sum_{n=1}^{\infty}A_n=\infty.
\]
\end{proof}

This corollary sharply separates two persistent regimes. A system may converge to a positive limiting burden and remain pointwise stable, but such a system necessarily accumulates infinite total burden over an infinite horizon. By contrast, finite cumulative burden requires the stronger regime \(A_n\to 0\). Since \(A_n\) already includes redundancy discounting, this distinction concerns the effective structural burden that remains after overlap has been taken into account, not merely a raw uncorrected penalty total.

\subsection{No explosive cycle-aging}

A central claim of the framework is that indefinite repetition does not force pointwise age explosion. This is made precise in the following result.

\begin{theorem}[No explosive cycle-aging]
Under the redundancy-adjusted cycle-level age model, the sequence \(\{A_n\}\) cannot diverge to \(+\infty\).
\end{theorem}

\begin{proof}
By Theorem 1,
\[
0\le A_n<B_\varepsilon \qquad \text{for every } n.
\]
Hence, \(\{A_n\}\) is uniformly bounded above. A bounded real sequence cannot satisfy
\[
A_n\to +\infty.
\]
Therefore, explosive cycle-level aging is impossible under the present model.
\end{proof}

This theorem is conceptually decisive, even though it follows directly from boundedness. It shows that terminality cannot be identified with pointwise age explosion. If long-run deterioration occurs in this framework, it must be understood through accumulation, persistence of positive burden, or failure to attenuate---not through divergence of the cycle-level age itself to infinity. This remains true even though the severity channel may increase locally, because the total age process is structurally bounded by construction and discounted by the redundancy channel.
\subsection{Eventual monotonic stabilization}

Long-run convergence can also be obtained from a simpler structural condition: eventual monotonicity. The next theorem formalizes this idea.

\begin{theorem}[Eventual monotonic stabilization]
If there exists \(N_0\) such that for all \(n\ge N_0\),
\begin{equation}
A_{n+1}\ge A_n,
\end{equation}
then \(\{A_n\}\) converges. Likewise, if there exists \(N_0\) such that for all \(n\ge N_0\),
\begin{equation}
A_{n+1}\le A_n,
\end{equation}
then \(\{A_n\}\) also converges.
\end{theorem}

\begin{proof}
First suppose that
\[
A_{n+1}\ge A_n \qquad \text{for all } n\ge N_0.
\]
Then the sequence is eventually nondecreasing. By Theorem 1, it is bounded above by \(B_\varepsilon\). Hence, convergence follows from the monotone convergence theorem for real sequences.

Now suppose that
\[
A_{n+1}\le A_n \qquad \text{for all } n\ge N_0.
\]
Then the sequence is eventually nonincreasing. Again by Theorem 1, it is bounded below by \(0\). Therefore, convergence follows in this case as well.
\end{proof}

This result shows that once the age process enters a one-sided trend, boundedness alone forces asymptotic settling. In other words, if the system eventually ages only upward or only downward from some point onward, then the age sequence cannot wander indefinitely; it must converge. At the level of interpretation, this means that once the combined effect of severity and redundancy adjustment becomes eventually monotone, the total effective burden must stabilize.

\subsection{Zero-burden characterization under nondegenerate redundancy}

The strongest persistence regime is the one in which cycle-level burden tends to zero. The next theorem shows exactly what this means at the component level, provided redundancy does not asymptotically erase the effective contribution of active components.

\begin{theorem}[Zero-burden characterization under nondegenerate redundancy]
Assume that for every active component \(i\) with \(w_i>0\), there exist constants \(\eta_i>0\) and \(N_i\in\mathbb{N}\) such that
\begin{equation}
1-R_{n,i}\ge \eta_i \qquad \text{for all } n\ge N_i.
\end{equation}
Then
\begin{equation}
A_n\to 0 \quad \Longleftrightarrow \quad x_{n,i}\to 1 \qquad \text{for every active component } i.
\end{equation}
\end{theorem}

\begin{proof}
First suppose that
\[
x_{n,i}\to 1 \qquad \text{for every active component } i.
\]
Since \(\phi\) is continuous on \((0,1]\) and
\[
\phi(1)=0,
\]
it follows that
\[
\phi(x_{n,i})\to 0.
\]
Also, because
\[
0\le 1-R_{n,i}\le 1,
\]
one has
\[
0\le w_i(1-R_{n,i})\phi(x_{n,i})\le w_i\phi(x_{n,i})\to 0.
\]
Since the number of components is finite, summation yields
\[
A_n\to 0.
\]
This proves the forward implication.

Conversely, suppose that
\[
A_n\to 0,
\]
but that for some active component \(i_0\), the convergence
\[
x_{n,i_0}\to 1
\]
does not hold. Then there exist a constant \(c>0\) and a subsequence \(n_k\) such that
\[
x_{n_k,i_0}\le 1-c \qquad \text{for all } k.
\]
Because \(\phi\) is strictly decreasing,
\[
\phi(x_{n_k,i_0})\ge \phi(1-c)>0.
\]
Also, for all sufficiently large \(k\),
\[
1-R_{n_k,i_0}\ge \eta_{i_0}>0.
\]
Therefore,
\[
A_{n_k} \ge w_{i_0}(1-R_{n_k,i_0})\phi(x_{n_k,i_0}) \ge w_{i_0}\eta_{i_0}\phi(1-c)>0,
\]
which contradicts the assumption that
\[
A_n\to 0.
\]
Hence,
\[
x_{n,i}\to 1 \qquad \text{for every active component } i.
\]
\end{proof}

This theorem provides the strongest structural interpretation of zero-burden persistence. Under nondegenerate redundancy, vanishing long-run age is equivalent to asymptotically perfect consistency in every active component. In this sense, the zero-burden regime is not a vague global phenomenon; it has a precise componentwise meaning. It also shows that the redundancy channel cannot by itself manufacture zero-burden persistence in the limit if the effective non-redundant share remains bounded away from zero for active components.
\subsection{Persistent imperfection excludes zero-burden persistence}

The contrapositive form of Theorem 10 yields an equally important exclusion result.

\begin{corollary}[Persistent imperfection excludes zero-burden persistence]
If there exist an active component \(i_0\), constants \(c>0\), \(\eta>0\), and infinitely many indices \(n\) such that
\begin{equation}
x_{n,i_0}\le 1-c, \qquad 1-R_{n,i_0}\ge \eta,
\end{equation}
then
\begin{equation}
A_n\not\to 0.
\end{equation}
\end{corollary}

\begin{proof}
The conclusion follows immediately from the contrapositive form of Theorem 10.
\end{proof}

This corollary is useful because it gives a simple diagnostic criterion. If some active component remains persistently imperfect and its burden is not asymptotically cancelled by redundancy, then the system cannot enter the zero-burden regime.

\subsection{Synthesis of long-run asymptotic regimes}

The results above allow the framework to move from static age estimation to a genuine theory of repeated artificial persistence. The age sequence is uniformly bounded at every cycle, so explosive pointwise aging is excluded. Under finite total variation, or under the stronger geometric damping condition, the age process converges to a finite limit and positive marginal aging vanishes. If that limit is strictly positive, the system remains in a burdened persistent regime and necessarily accumulates infinite cumulative burden over an infinite horizon. If instead the age sequence converges to zero, then the system enters the stronger zero-burden persistent regime. Under nondegenerate redundancy, this regime is equivalent to asymptotically perfect consistency in all active components.

Accordingly, the central asymptotic hierarchy of the framework is not organized around unrestricted pointwise age growth, because such growth is excluded by construction. Rather, it is organized around distinct long-run burden profiles:
\[
\begin{aligned}
A_n &\text{ bounded}, &&\text{which excludes explosive cycle-aging,} \\
A_n &\to A^\ast \in (0,B_\varepsilon), &&\text{which yields burdened persistence,} \\
A_n &\to 0, &&\text{which yields zero-burden persistence,} \\
\mathcal{C}_N &\to \infty, &&\text{which captures cumulative terminal burden.}
\end{aligned}
\]

The main theoretical insight can therefore be stated clearly: within the proposed framework, indefinite cyclic continuation does not require unbounded structural aging. A system may pass through infinitely many cycles while its cycle-level age remains bounded, while its positive age increments vanish under stabilization conditions, and, in the strongest regime, while its cycle-level structural burden tends to zero. Long-run persistence is therefore not a problem of inevitable explosive deterioration, but a problem of how redundancy-adjusted structural burden behaves asymptotically under repeated operation.
\section{Discussion and Implications}

\subsection{Discussion}

The main contribution of the present framework is that it transforms the redundancy-adjusted Artificial Age Score from a static evaluative quantity into a mathematically analyzable age process for repeated artificial operation. In this respect, the framework is closer to cybernetics than to one-shot benchmarking, because its central object is not a single success event but the evolution of structural burden across cycles (Wiener, 1948/1961). It is also closer to reliability theory, because the relevant question is not isolated success alone but the persistence of organized performance under structural constraints (von Neumann, 1956).

A first implication of this shift is conceptual. Long-run artificial operation should not be described simply in binary terms such as ``working'' versus ``failing.'' A system may remain operational while converging to a positive asymptotic burden, or it may approach the stronger regime in which cycle-level burden vanishes. The theory therefore replaces a binary picture of persistence with a hierarchy of asymptotic burden profiles. This matters because repeated operation may remain formally stable even when a nonzero structural cost persists from cycle to cycle.

A second implication is mathematical. Because cycle-level age is uniformly bounded in the present framework, explosive pointwise aging is excluded by construction. This shifts the interpretation of terminality away from instantaneous divergence and toward cumulative burden, persistent positive load, or failure of attenuation. In other words, the central question is no longer whether the age value can blow up at a single cycle, but whether structural burden stabilizes, oscillates, remains positive, or accumulates across an unbounded horizon. This distinction is consistent with the broader lesson that indefinite processes require explicit structural discipline if they are to remain mathematically meaningful (Hilbert, 1926). The same general point is reinforced by the foundational warning that unrestricted totalization can lead to contradiction rather than rigor (Russell, 1903). It is also consistent with the axiomatic demand that formally meaningful reasoning proceed under explicit constraint (Zermelo, 1908).

A third implication concerns redundancy. In many artificial systems, apparent weaknesses are not independent; rather, they may be partially overlapping manifestations of the same underlying instability. A purely additive burden measure would therefore risk overcounting deterioration. By incorporating the factor \((1-R_{n,i})\), the present framework interprets redundancy as a correction against structural double-counting rather than as a superficial adjustment. This makes the age functional more appropriate for systems in which organization and compensation affect the effective significance of local weakness (von Neumann, 1956). It is also broadly consistent with the information-theoretic lesson that structure matters, rather than raw quantity alone, in the interpretation of uncertainty and degradation (Shannon, 1948).

A fourth implication concerns long-run stabilization. The present theory shows that boundedness alone is not the most informative property of repeated operation. More refined distinctions emerge once one studies convergence, finite variation, geometric damping, and zero-burden persistence. In particular, the framework distinguishes a system that merely remains bounded from one whose positive marginal age increments vanish, and further from one whose cycle-level burden itself converges to zero. This gives the theory a stronger analytical role than that of a descriptive score, because it identifies mathematically distinct forms of long-run persistence.

A fifth implication concerns scope. The present framework does not claim to explain intelligence, semantics, or cognition in full. Nor does it claim that repeated successful operation is equivalent to understanding. Its contribution is narrower and more defensible: it provides a formal account of how structural burden behaves across repeated cycles under explicit assumptions on consistency, redundancy, and weighting. This restricted scope is a strength rather than a weakness, because it keeps the theory aligned with what its formal structure can genuinely justify (Bender \& Koller, 2020). It is also consistent with the more general caution that formal systems should not be mistaken for exhaustive accounts of all that they touch (Gödel, 1931).

Taken together, these points suggest that the strongest interpretation of the framework is neither purely metaphorical nor excessively ambitious. It is not merely a new score, because it generates a mathematically structured theory of asymptotic burden under repeated operation. At the same time, it is not a complete theory of intelligence or cognition. Its contribution lies in a more disciplined middle ground: a bounded, redundancy-aware, cycle-level formalism for analyzing long-run artificial persistence.

\subsection{Implications}

One practical implication is for the evaluation of systems that operate repeatedly rather than once. When a model is updated, queried, adapted, or exposed to continuing interaction, endpoint performance alone may fail to reveal whether the underlying structural burden is stabilizing or compounding. The present framework therefore suggests that longitudinal evaluation should complement one-shot benchmarking whenever repeated deployment is central to the application domain (NIST, 2023). The broader cybernetic emphasis on feedback and continued regulation supports the importance of evaluating systems across ongoing operation rather than only at isolated endpoints (Wiener, 1948/1961).

A second implication is for reliability-oriented system design. Modern AI systems often involve coupled modules, fallback mechanisms, retrieval components, memory layers, and post-processing structures. In such settings, observed weaknesses may be correlated rather than independent. A redundancy-adjusted age functional is therefore more appropriate than a raw additive failure count, because it discounts overlapping burden and better reflects the effective structural load carried by the overall system (von Neumann, 1956).

A third implication is for continual adaptation and memory-aware operation. In systems that repeatedly incorporate new information, structural burden should not be interpreted only as immediate error or failure. It may also reflect the cost of retaining, updating, compressing, or reorganizing prior structure over time. Under this interpretation, the distinction between burdened persistence and zero-burden persistence becomes relevant for understanding whether repeated adaptation is merely sustainable or asymptotically self-stabilizing (Parisi et al., 2019). This analogy may also be read alongside broader memory-theoretic discussions of retention and adaptive forgetting, even though the present framework does not model human memory directly (Schacter, 1999). The distinction between different forms of retention can also be related, at an interpretive level, to Tulving's account of memory organization (Tulving, 1972). Tulving's later work on memory systems is likewise relevant as a conceptual background for thinking about differentiated retention across time (Tulving, 1985).

A fourth implication is for monitoring and governance. The distinction between signed change \(\delta_n\), absolute fluctuation \(D_n\), positive accumulation \(\Delta_n\), and cumulative burden \(\mathcal{C}_N\) provides a more differentiated vocabulary for tracking long-run system behavior. This may be useful in settings where acceptable one-time performance is not sufficient, and where repeated interaction changes the system's risk profile over time. In such settings, the relevant issue is not only whether the system is currently acceptable, but whether its burden is stabilizing, oscillating, or accumulating across continued use (NIST, 2023). This broader concern with trustworthiness, explicability, and responsible oversight is also consistent with contemporary ethical frameworks for AI governance (Floridi et al., 2018).

A fifth implication is for future theoretical development. Because the present framework already establishes boundedness, ordering, perturbation sensitivity, convergence under finite variation, geometric stabilization, and zero-burden characterization under nondegenerate redundancy, it provides a basis for later extensions to stochastic settings, distributional drift, adaptive control, or empirical continual-learning scenarios. Such extensions would not replace the current theory, but would build on its formal core.

Overall, the practical importance of the framework lies in its ability to distinguish forms of persistence that would otherwise remain collapsed into a single notion of ``continued operation.'' A system may continue with stable positive burden, continue with oscillatory but bounded burden, or continue while attenuating its structural burden toward zero. The present theory provides the formal language needed to separate these possibilities.
\section{Conclusion}

The present study developed a long-run persistence framework for AI systems under the redundancy-adjusted Artificial Age Score (Kayadibi, 2026). Its central contribution was to extend AAS from a static evaluative quantity into a cycle-level formalism for repeated artificial operation. Within this framework, the main theoretical question was whether indefinite cyclic continuation must imply unbounded structural aging.

The results show that this is not the case. The cycle-level age functional is well defined, nonnegative, and uniformly bounded, so explosive pointwise aging is excluded by construction. This boundedness makes it possible to analyze repeated operation over an unbounded horizon without requiring the age assigned to individual cycles to diverge. In this sense, the framework formally separates indefinite continuation from inevitable pointwise age explosion.

On this basis, the study established a hierarchy of long-run asymptotic regimes. A system may remain in a burdened persistent regime, in which cycle-level age converges to a positive limit, or it may enter the stronger zero-burden persistent regime, in which cycle-level burden converges to zero. The theory also showed that finite total variation implies convergence of the age sequence, that geometrically damped componentwise perturbations yield geometric stabilization of age, and that finite cumulative burden necessarily requires vanishing cycle-level age. Under nondegenerate redundancy, the strongest regime admits a precise componentwise interpretation: zero-burden persistence is equivalent to asymptotically perfect consistency in all active components.

These results clarify that long-run persistence is not a single condition. A system may continue operating while carrying a stable positive burden, while oscillating within a bounded range, or while asymptotically attenuating its burden toward zero. The theoretical significance of the framework lies in making these possibilities formally distinguishable. Persistence is therefore not adequately described by the binary question of whether a system still functions, but by the asymptotic profile of the structural burden it carries across repeated cycles.

At the same time, the scope of the framework is deliberately limited. The present model does not claim to provide a full theory of intelligence, cognition, or semantic understanding. Its contribution is narrower and more rigorous: it supplies a bounded, redundancy-aware, cycle-level account of structural burden under repeated operation. This restricted scope is a strength because it aligns the claims of the theory with what the formal results actually establish.

The central conclusion may therefore be stated as follows: within the proposed framework, indefinite cyclic continuation is mathematically compatible with bounded structural age and, under stronger stabilization conditions, with vanishing marginal aging and even asymptotically vanishing cycle-level burden. Long-run artificial persistence should therefore be understood not as a process of inevitable explosive deterioration, but as a formally analyzable problem of asymptotic structural burden under repeated operation.
\section*{References}

\begingroup
\setlength{\parindent}{0pt}
\setlength{\parskip}{\baselineskip}

Aristotle. (1998). \textit{Metaphysics} (H. Lawson-Tancred, Trans.). Penguin Books. (Original work written ca. 350 B.C.E.)

Ashby, W. R. (1960). \textit{Design for a brain: The origin of adaptive behaviour} (2nd ed.). Chapman \& Hall.

Bender, E. M., \& Koller, A. (2020). Climbing towards NLU: On meaning, form, and understanding in the age of data. In \textit{Proceedings of the 58th Annual Meeting of the Association for Computational Linguistics} (pp. 5185--5198). Association for Computational Linguistics. \url{https://doi.org/10.18653/v1/2020.acl-main.463}

Bird, A. (2004). Thomas Kuhn. In E. N. Zalta (Ed.), \textit{The Stanford Encyclopedia of Philosophy}. Metaphysics Research Lab, Stanford University. \url{https://plato.stanford.edu/archives/win2004/entries/thomas-kuhn/}

Bolzano, B. (1851). \textit{Paradoxien des Unendlichen}. C. H. Reclam.

Cantor, G. (1874). Ueber eine Eigenschaft des Inbegriffes aller reellen algebraischen Zahlen. \textit{Journal f\"ur die Reine und Angewandte Mathematik, 77}, 258--262. \url{https://doi.org/10.1515/crll.1874.77.258}

Cohen, S. M. (2000). Aristotle's metaphysics. In E. N. Zalta (Ed.), \textit{The Stanford Encyclopedia of Philosophy}. Metaphysics Research Lab, Stanford University. \url{https://plato.stanford.edu/entries/aristotle-metaphysics/}

Cook, R. T. (2023). Frege's logic. In E. N. Zalta \& U. Nodelman (Eds.), \textit{The Stanford Encyclopedia of Philosophy} (Fall 2023 ed.). Metaphysics Research Lab, Stanford University. \url{https://plato.stanford.edu/archives/fall2023/entries/frege-logic/}

Dedekind, R. (1888). \textit{Was sind und was sollen die Zahlen?} Friedrich Vieweg \& Sohn.

Floridi, L., Cowls, J., Beltrametti, M., Chatila, R., Chazerand, P., Dignum, V., Luetge, C., Madelin, R., Pagallo, U., Rossi, F., Schafer, B., Valcke, P., \& Vayena, E. (2018). AI4People---An ethical framework for a good AI society: Opportunities, risks, principles, and recommendations. \textit{Minds and Machines, 28}(4), 689--707. \url{https://doi.org/10.1007/s11023-018-9482-5}

Frege, G. (1967). \textit{Begriffsschrift, a formula language, modeled upon that of arithmetic, for pure thought} (S. Bauer-Mengelberg, Trans.). In J. van Heijenoort (Ed.), \textit{From Frege to G\"odel: A source book in mathematical logic, 1879--1931} (pp. 1--82). Harvard University Press. (Original work published 1879)

Frege, G. (1980). \textit{The foundations of arithmetic: A logico-mathematical enquiry into the concept of number} (J. L. Austin, Trans., 2nd rev. ed.). Northwestern University Press. (Original work published 1884)

G\"odel, K. (1931). \"Uber formal unentscheidbare S\"atze der Principia Mathematica und verwandter Systeme I. \textit{Monatshefte f\"ur Mathematik und Physik, 38}(1), 173--198. \url{https://doi.org/10.1007/BF01700692}

Hilbert, D. (1926). \"Uber das Unendliche. \textit{Mathematische Annalen, 95}(1), 161--190. \url{https://doi.org/10.1007/BF01206605}

Kant, I. (1992). \textit{Critique of pure reason} (N. Kemp Smith, Trans.). Macmillan. (Original work published 1781/1787)

Kayadibi, S. Y. (2026). Redundancy-as-masking: Formalizing the Artificial Age Score (AAS) to model memory aging in generative AI. \textit{Frontiers in Artificial Intelligence, 9}, Article 1732691. \url{https://doi.org/10.3389/frai.2026.1732691}

Kuhn, T. S. (2012). \textit{The structure of scientific revolutions} (4th ed.). University of Chicago Press. (Original work published 1962)

McCarthy, J., Minsky, M. L., Rochester, N., \& Shannon, C. E. (1955, August 31). \textit{A proposal for the Dartmouth summer research project on artificial intelligence} [Research proposal]. Dartmouth College.

McCulloch, W. S., \& Pitts, W. (1943). A logical calculus of the ideas immanent in nervous activity. \textit{The Bulletin of Mathematical Biophysics, 5}(4), 115--133. \url{https://doi.org/10.1007/BF02478259}

Newell, A., \& Simon, H. A. (1976). Computer science as empirical inquiry: Symbols and search. \textit{Communications of the ACM, 19}(3), 113--126. \url{https://doi.org/10.1145/360018.360022}

Parisi, G. I., Kemker, R., Part, J. L., Kanan, C., \& Wermter, S. (2019). Continual lifelong learning with neural networks: A review. \textit{Neural Networks, 113}, 54--71. \url{https://doi.org/10.1016/j.neunet.2019.01.012}

Rosenblatt, F. (1958). The perceptron: A probabilistic model for information storage and organization in the brain. \textit{Psychological Review, 65}(6), 386--408. \url{https://doi.org/10.1037/h0042519}

Russell, B. (1903). \textit{The principles of mathematics}. Cambridge University Press.

Schacter, D. L. (1999). The seven sins of memory: Insights from psychology and cognitive neuroscience. \textit{American Psychologist, 54}(3), 182--203. \url{https://doi.org/10.1037/0003-066X.54.3.182}

Shabel, L. (2013). Kant's philosophy of mathematics. In E. N. Zalta (Ed.), \textit{The Stanford Encyclopedia of Philosophy} (Fall 2013 ed.). Metaphysics Research Lab, Stanford University. \url{https://plato.stanford.edu/archives/fall2013/entries/kant-mathematics/}

Shannon, C. E. (1948). A mathematical theory of communication. \textit{Bell System Technical Journal, 27}(3), 379--423. \url{https://doi.org/10.1002/j.1538-7305.1948.tb01338.x}

Tabassi, E. (2023). \textit{Artificial intelligence risk management framework (AI RMF 1.0)} (NIST AI 100-1). National Institute of Standards and Technology. \url{https://doi.org/10.6028/NIST.AI.100-1}

Tulving, E. (1972). Episodic and semantic memory. In E. Tulving \& W. Donaldson (Eds.), \textit{Organization of memory} (pp. 381--403). Academic Press.

Tulving, E. (1985). Memory and consciousness. \textit{Canadian Psychology / Psychologie canadienne, 26}(1), 1--12. \url{https://doi.org/10.1037/h0080017}

Turing, A. M. (1950). Computing machinery and intelligence. \textit{Mind, 59}(236), 433--460. \url{https://doi.org/10.1093/mind/LIX.236.433}

von Neumann, J. (1956). Probabilistic logics and the synthesis of reliable organisms from unreliable components. In C. E. Shannon \& J. McCarthy (Eds.), \textit{Automata studies} (pp. 43--98). Princeton University Press.

Wiener, N. (1961). \textit{Cybernetics: Or control and communication in the animal and the machine} (2nd ed.). MIT Press. (Original work published 1948)

Zermelo, E. (1908). Untersuchungen \"uber die Grundlagen der Mengenlehre. I. \textit{Mathematische Annalen, 65}(2), 261--281. \url{https://doi.org/10.1007/BF01449999}

\endgroup
\end{document}